\documentclass{article}

\usepackage[preprint]{neurips_2026}

\usepackage{marvosym}
\usepackage[utf8]{inputenc} 
\usepackage[T1]{fontenc}    
\usepackage{hyperref}       
\usepackage{url}            
\usepackage{booktabs}      
\usepackage{amsfonts}    
\usepackage{nicefrac}    
\usepackage{microtype}     
\usepackage{xcolor}        

\usepackage{amsmath}
\usepackage{amsthm}
\theoremstyle{plain}
\newtheorem{lemma}{Lemma}
\newtheorem{proposition}{Proposition}
\usepackage{cleveref}
\usepackage{multirow}
\usepackage{tcolorbox}

\setcitestyle{numbers,square,compress}
\tcbuselibrary{skins,breakable}

\title{Balancing Performance and Diversity in GRPO Autoregressive Text-to-Image Post-Training}

\author{%
Yuanhao Chiang$^{1}$,
Hongbo Duan$^{1}$,
Chunru Yang$^{1}$,
Jiahua Pei$^{1}$,
Yi Liu$^{1}$,
Xueqian Wang$^{1}\textsuperscript{\Letter}$ \\
$^{1}$ Shenzhen International Graduate School, Tsinghua University \\
\texttt{jiang-yh24@mails.tsinghua.edu.cn; wang.xq@sz.tsinghua.edu.cn}
}

\begin{document}

\maketitle

\begin{abstract}
  Autoregressive text-to-image (T2I) generation has recently advanced rapidly, yet aligning generated images with human preferences remains challenging. GRPO-style online reinforcement learning provides an effective framework; however, existing methods typically treat reference-policy divergence as fixed, despite its direct impact on policy optimization. We study this overlooked factor within a unified f-divergence framework, encompassing forward KL, reverse KL, and JS divergence, for GRPO-style autoregressive T2I alignment. Our systematic theoretical analysis reveals that different divergences reshape token-level updates in distinct ways. In particular, under the sampled-token shaping form used, JS regularization achieves a favorable trade-off by mitigating uniform bias relative to the reference policy while still discouraging large deviations. Extensive experiments on LlamaGen and Janus-7B show that JS divergence achieves the strongest or highly competitive optimization performance on most evaluation metrics while maintaining favorable generation diversity. The code is available at \url{https://github.com/tuoyou-hao/BPD-GRPO}.

\end{abstract}

\section{Introduction}

Autoregressive text-to-image (T2I) generation has emerged as a promising alternative to diffusion-based methods, representing images as sequences of discrete visual tokens conditioned on text prompts. Leveraging large-scale multimodal pretraining \cite{song2026cologen} and powerful autoregressive backbones, these models generate visually rich images while inheriting the next-token prediction paradigm proven effective in language modeling. However, strong generative capability alone does not guarantee preference-aligned generation. In practical T2I systems, generated images should not only align with the input prompt but also reflect human preferences. Prior work has investigated various post-training approaches for diffusion models \cite{fan2023dpok,wallace2024diffusion,yang2024using,liang2024step,li2024aligning}, including positive-signal, consistency-aware, and entropy-aware preference alignment for T2I generation ~\cite{sun2025positive,sun2025identical,bai2025entropy}.
For visual generation \cite{zeng2025mred,yin2026ai}, online GRPO-style optimization \cite{shao2024deepseekmath,zhang2025group,fang2026viss,sun2025reinforcement} is particularly appealing, as it provides a natural framework for group-wise sampling, relative reward comparison, and token-level credit assignment over discrete image tokens. Nevertheless, although recent studies have extended preference optimization and GRPO-style reinforcement post-training to broader alignment and visual generation settings~\cite{sun2024generalizing,lv2025hidden,luo2026reinforcement,luo2025sample,sun2026power}, existing autoregressive T2I formulations still largely regard the reference-policy divergence as a fixed design choice, specifically adopting the KL penalty without exploring alternative divergence formulations.

\section{Related Work}
\subsection{Policy Optimization for Text-to-Image Generation}
Policy optimization aims to adapt generative models toward human preferences beyond standard likelihood or reconstruction objectives. Early methods \cite{fan2023dpok} commonly rely on reward models and reinforcement learning \cite{xia2025delay,suncalibration,yin2025floorplan} to improve text-image matching, aesthetic quality, and human preference scores. More recent diffusion-based methods formulate alignment as direct preference optimization or preference-aware denoising, avoiding explicit online RL while still improving reward-model evaluation and consistency with human preferences \cite{wallace2024diffusion,yang2024using,liang2024step,li2024aligning}. Subsequent studies further improve T2I preference alignment from positive-signal utilization, preference consistency, and entropy-aware diversity perspectives~\cite{sun2025positive,sun2025identical,bai2025entropy}.
These methods primarily focus on diffusion backbones and differ in preference data construction, denoising-step weighting, or offline objective design.
In contrast, our study investigates reinforcement learning for autoregressive T2I models, where image generation is performed by discrete visual-token prediction.

\subsection{Online RLHF and GRPO-Style Optimization}
Online RLHF typically optimizes the policy against a reward signal while constraining its deviation from the reference policy. PPO \cite{schulman2017proximal} and related trust-region methods are widely utilized for language-model alignment as they balance reward maximization and policy stability \cite{ouyang2022training}. GRPO-style objectives further simplify online alignment by using group-wise relative comparison instead of a separate value model, making them suitable for scalable policy optimization \cite{shao2024deepseekmath}. For autoregressive image generation, GCPO\cite{zhang2025group} extends this paradigm to discrete visual tokens by assigning group-wise critical-token rewards and applying token-level updates. Our study builds on this line of work, but focuses on an overlooked component: the divergence regularization used to constrain the aligned policy. Rather than treating it as a default KL penalty, we examine how different divergence forms affect the optimization dynamics of autoregressive T2I alignment.

\subsection{\texorpdfstring{$f$}{f}-Divergence Regularization}
The $f$-divergence family \cite{csiszar1967information,kullback1951information,lin1991divergence} provides a unified view of many distributional discrepancy measures, including forward KL, reverse KL, JS divergence, and so on. Divergence choice is known to affect mode-seeking and mode-covering behavior in generative modeling and preference optimization~\cite{goodfellow2014generative,go2023aligning,wang2024beyond,sun2024generalizing}. Recent work has also introduced $f$-divergence perspectives into T2I post-training, showing that alternative divergences can improve the performance-diversity trade-off in offline diffusion alignment \cite{sun2025generalizing}. Our work differs by investigating divergence regularization as an online trust-region component for autoregressive T2I alignment. We connect the divergence form to token-level gradient shaping under GRPO-style optimization and empirically compare the performance of forward KL, reverse KL, and JS divergence under the same training and evaluation settings.

\section{Method}
\begin{figure}[t]
\centering
\includegraphics[width=0.9\textwidth]{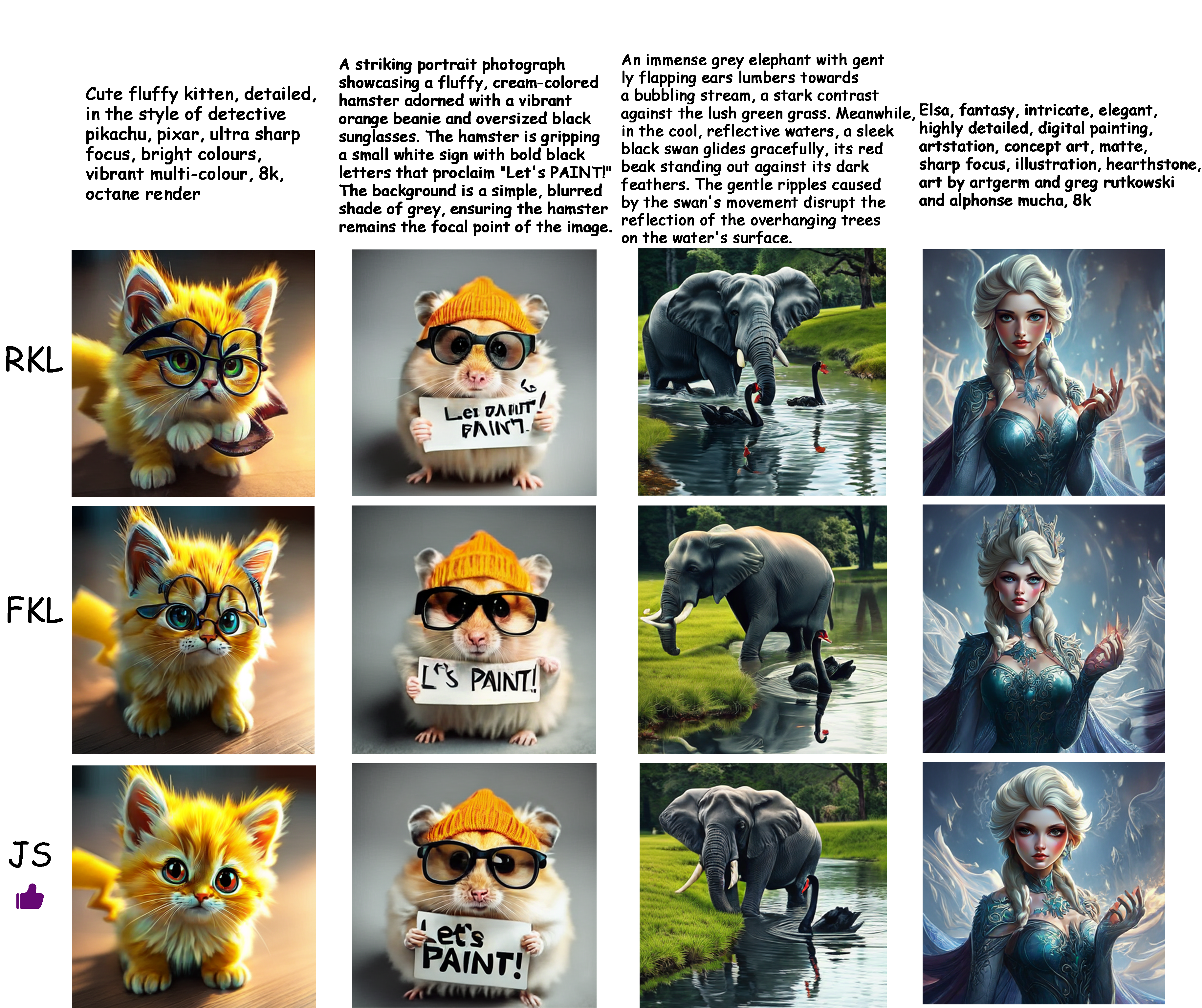}
\caption{Qualitative comparison of Janus-7B variants trained with RKL divergence, FKL divergence and JS divergence. Columns share the same prompt and rows correspond to divergence variants.}
\label{fig:qualitative_comparison_janus}
\end{figure}
\begin{figure}[t]
\centering
\includegraphics[width=0.9\textwidth]{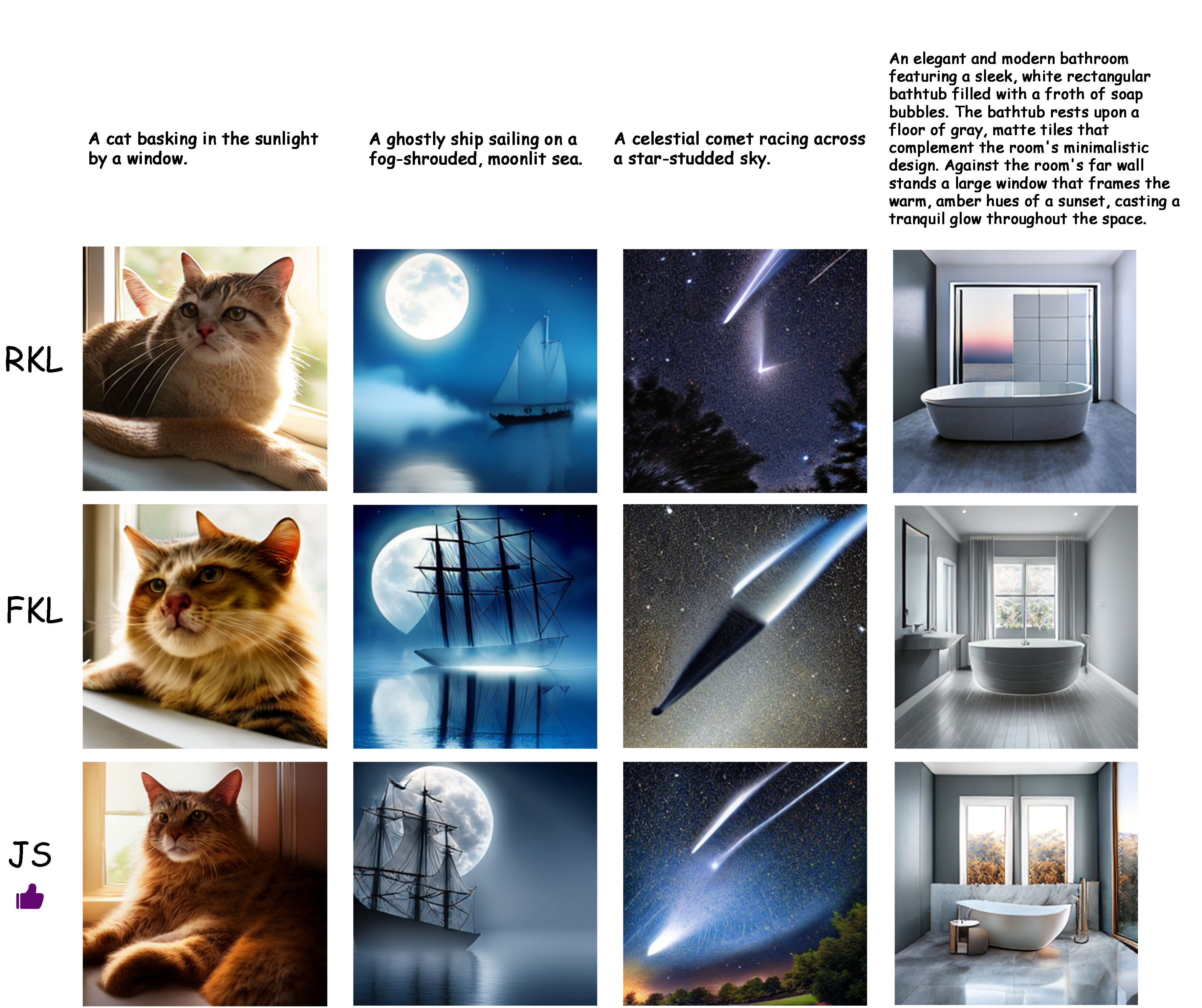}
\caption{Qualitative comparison of LlamaGen variants trained with RKL divergence, FKL divergence, and JS divergence. Columns share the same prompt and rows correspond to divergence variants.}
\label{fig:qualitative_comparison_llamagen}
\end{figure}
We study policy optimization for autoregressive T2I models in a GRPO-style online optimization framework. Given a prompt, the policy samples a group of candidate image-token sequences, computes relative advantages from reward scores, and updates token probabilities with a reference-policy regularization. 

\subsection{Preliminaries}
We first introduce the standard $f$-divergence formulation. Let $p$ and $q$ denote two distributions over the same action space. The $f$-divergence between $q$ and $p$ is defined as:

\begin{equation}
\mathbb{D}_f(p\Vert q)=\mathbb{E}_{a\sim q}\left[f\!\left(\frac{p(a)}{q(a)}\right)\right],
\label{eq:fdiv_method_final}
\end{equation}

\noindent satisfying $f(1)=0$ and $f''(x)>0$. Thus, for discrete action space $\mathcal{A}$, we have:

\begin{equation}
\mathbb{D}_f(p\Vert q)=\sum\limits_{a\in\mathcal{A}}q(a)f\!\left(\frac{p(a)}{q(a)}\right).
\label{eq:fdiv_discrete_method}
\end{equation}

In our settings, $p$ corresponds to the current policy $\pi_\theta$, while $q$ corresponds to the frozen reference policy $\pi_{\text{ref}}$; the RKL variant corresponds to the standard GRPO-style reference-policy regularization baseline, while FKL and JS are obtained by replacing only the divergence-induced shaping term. Under this convention, $f(t)=t\log t$ corresponds to $\mathrm{KL}(\pi_\theta\Vert\pi_{\mathrm{ref}})$, which we refer to as reverse KL, while $f(t)=-\log t$ corresponds to $\mathrm{KL}(\pi_{\mathrm{ref}}\Vert\pi_\theta)$, which we refer to as forward KL. We compare a total of three representative choices from the $f$-divergence family:

\begin{tcolorbox}[halign=flush left]
Reverse KL: \quad $f_{\mathrm{RKL}}(t) = t\log t$, \quad
$f'_{\mathrm{RKL}}(t) = \log t + 1$ \\

Forward KL: \quad $f_{\mathrm{FKL}}(t) = -\log t$, \quad
$f'_{\mathrm{FKL}}(t) = -\frac{1}{t}$ \\

Jensen-Shannon: \quad
$f_{\mathrm{JS}}(t) = t \log \frac{2t}{1+t} + \log \frac{2}{1+t}$, \quad
$f'_{\mathrm{JS}}(t) = \log \frac{2t}{1+t}$
\label{eq:three_divergences_final}
\end{tcolorbox}

Then, we formalize the autoregressive T2I alignment setting. For a given prompt $x$, the current policy $\pi_\theta$ generates a group of $G$ image-token sequences:

\begin{equation}
\mathcal{Y}(x)=\{y^{(g)}\}_{g=1}^{G},\qquad y^{(g)}\sim \pi_\theta(\cdot|x).
\label{eq:group_sampling_method}
\end{equation}
These samples are evaluated by a reward model, based on which group-wise relative advantages are computed to facilitate token-level policy updates. To mitigate policy drift, a frozen reference policy $\pi_{\text{ref}}$ is employed for regularization. We maintain a consistent sampling strategy, advantage computation, and all other optimization schedules across all divergence variants. 

Finally, for a token $y_t^{(g)}$, we define the token-level importance ratio as $w_t^{(g)}=\pi_\theta(y_t^{(g)}|x,y_{<t}^{(g)})/\pi_{\text{ref}}(y_t^{(g)}|x,y_{<t}^{(g)})$, which quantifies the policy deviation from the reference model at each autoregressive generation step. Let $A_t^{(g)}$ denote the token-level advantage. In our implementation, we broadcast the sequence-level group-relative reward to each valid image token:

\begin{equation}
A_t^{(g)}=m_t^{(g)}\cdot \frac{R^{(g)}-\mathrm{mean}_{j}R^{(j)}}{\mathrm{std}_{j}R^{(j)}+\epsilon},
\end{equation}
where $m_t^{(g)}$ represents the visual-token mask. We can formulate the objective as:

\begin{equation}
\mathcal{J}(\theta)=\mathbb{E}_{x,\,y\sim\pi_\theta(\cdot|x)}\left[R(x,y)-\eta\sum_{t=1}^{T}\mathbb{D}_f\big(\pi_{\theta}(\cdot|x,y_{<t})\Vert\pi_{\text{ref}}(\cdot|x,y_{<t})\big)\right],
\label{eq:regularized_objective_method}
\end{equation}
where $x$ denotes the text prompt, $y$ represents the generated image-token sequence of length $T$, $R(x,y)$ is the sequence-level reward, and $\eta$ serves as the hyperparameter controlling the strength of the $f$-divergence regularization. In this formulation, the reward signal is converted into group-wise relative advantages, while the divergence term penalizes policy drift from $\pi_{\text{ref}}$.

\subsection{Gradient Analysis}

In this section, we further analyze how the choice of $f$-divergence affects token-level policy updates. The gradient of \Cref{eq:regularized_objective_method} can be decomposed as:

\begin{equation}
\nabla_{\theta}\mathcal{J}=\nabla_{\theta}\mathbb{E}_{x,\,y\sim\pi_\theta(\cdot|x)}\left[R(x,y)\right]-\eta\cdot\nabla_{\theta}\mathbb{E}_{x}\sum_{t=1}^{T}\mathbb{D}_f\big(\pi_{\theta}(\cdot|x,y_{<t})\Vert\pi_{\text{ref}}(\cdot|x,y_{<t})\big).
\label{eq:gradient_method}
\end{equation}
Firstly, we introduce \Cref{Policy Gradient Theorem}, showing formulation of policy gradient~\cite{sutton1999policy,thomas2017policy}.

\begin{lemma}[Policy Gradient Theorem]
\label{Policy Gradient Theorem}
For autoregressive generation, the policy gradient can be formulated as:
\begin{equation}
\nabla_{\theta}\mathbb{E}_{x,\,y\sim\pi_\theta(\cdot|x)}\left[R(x,y)\right]=\mathbb{E}_{x,\,y\sim\pi_\theta(\cdot|x)}\left[\sum_{t=1}^{T}\nabla_{\theta}\log\pi_{\theta}(y_t|x,y_{<t})A_t\right].
\label{eq:policy_gradient_method}
\end{equation}
\end{lemma}

Furthermore, let's consider the latter part of \Cref{eq:gradient_method}. Since the GRPO paradigm is formulated under the trust-region and on-policy framework, the state visitation distribution change is negligible within each update. Thus,

\begin{equation}
\begin{split}
&\nabla_{\theta}\mathbb{E}_{x}\sum_{t=1}^{T}\mathbb{D}_f\big(\pi_{\theta}(\cdot|x,y_{<t})\Vert\pi_{\text{ref}}(\cdot|x,y_{<t})\big)\\ =& \mathbb{E}_{x}\sum_{t=1}^{T}\nabla_{\theta}\mathbb{D}_f\big(\pi_{\theta}(\cdot|x,y_{<t})\Vert\pi_{\text{ref}}(\cdot|x,y_{<t})\big).
\label{eq:state_dist_approx_method}
\end{split}
\end{equation}
For a fixed prefix $y_{<t}$, we have that:

\begin{equation}
\begin{aligned}
&\nabla_{\theta}\mathbb{D}_f\big(\pi_{\theta}(\cdot|x,y_{<t})\Vert\pi_{\text{ref}}(\cdot|x,y_{<t})\big) \\
=&\sum\limits_{v\in\mathcal{V}}\pi_{\text{ref}}(v|x,y_{<t})f'\!\left(\frac{\pi_{\theta}(v|x,y_{<t})}{\pi_{\text{ref}}(v|x,y_{<t})}\right)\cdot\nabla_{\theta}\frac{\pi_{\theta}(v|x,y_{<t})}{\pi_{\text{ref}}(v|x,y_{<t})}\\
=&\sum\limits_{v\in\mathcal{V}}f'\!\left(\frac{\pi_{\theta}(v|x,y_{<t})}{\pi_{\text{ref}}(v|x,y_{<t})}\right)\cdot\nabla_{\theta}\pi_{\theta}(v|x,y_{<t})\\
=&\sum\limits_{v\in\mathcal{V}}f'\!\left(\frac{\pi_{\theta}(v|x,y_{<t})}{\pi_{\text{ref}}(v|x,y_{<t})}\right)\cdot\pi_{\theta}(v|x,y_{<t})\nabla_{\theta}\log\pi_{\theta}(v|x,y_{<t})\\
=&\mathbb{E}_{v\sim\pi_{\theta}}\left[f'\!\left(\frac{\pi_{\theta}(v|x,y_{<t})}{\pi_{\text{ref}}(v|x,y_{<t})}\right)\cdot\nabla_{\theta}\log\pi_{\theta}(v|x,y_{<t})\right].
\end{aligned}
\label{eq:f_div_gradient_method}
\end{equation}
Further combine \Cref{eq:policy_gradient_method} and \Cref{eq:f_div_gradient_method}, we obtain:

\begin{equation}
\nabla_{\theta}\mathcal{J}=\mathbb{E}_{x,\,y\sim\pi_\theta(\cdot|x)}\Bigg[\sum_{t=1}^{T}\nabla_{\theta}\log\pi_{\theta}(y_t|x,y_{<t})\cdot\underbrace{\left(A_t-\eta f'\!\left(\frac{\pi_{\theta}(y_t|x,y_{<t})}{\pi_{\text{ref}}(y_t|x,y_{<t})}\right)\right)}_{\hat{A}_t}\Bigg].
\label{eq:gradient_merged_advantage_method}
\end{equation}
Specializing this objective to autoregressive image-token generation gives:

\begin{equation}
\nabla_\theta\mathcal{J}(\theta)=\mathbb{E}\left[\sum_{g=1}^{G}\sum_{t}\nabla_\theta\log\pi_\theta(y_t^{(g)}|x,y_{<t}^{(g)})\cdot\left(A_t^{(g)}-\eta f'(w_t^{(g)})\right)\right].
\label{eq:grad_main_final}
\end{equation}
Hence, the divergence term modulates the token-level update via a merged advantage: 

\begin{equation}
\hat{A}_t^{(g)}=A_t^{(g)}-\eta f'(w_t^{(g)}).
\label{eq:merged_adv_final}
\end{equation}
\Cref{eq:merged_adv_final} reveals the key insight of our method: identical GRPO-style group sampling and reward design, different divergences induce distinct optimization dynamics solely through the form of $f'(w)$. The three divergences therefore lead to different gradient-shaping effects. Specifically, we have that:

\begin{equation}
f'_{\mathrm{RKL}}(1)=1,\qquad f'_{\mathrm{FKL}}(1)=-1,\qquad f'_{\mathrm{JS}}(1)=0.
\label{eq:local_behavior_method}
\end{equation}

Therefore, under the sampled-token shaping formulation utilized in our implementation, JS introduces no constant offset at $w_t=1$, whereas the uncentered RKL generator retains a positive offset. This makes JS more balanced around the reference policy while still suppressing overly aggressive updates when $w_t>1$. Under this formulation, forward KL satisfies $f'_{\mathrm{FKL}}(x)=-1/x<0$ for all $x>0$, while reverse KL satisfies $f'_{\mathrm{RKL}}(x)=\log x+1$. Moreover, the following proposition derives comparative inequalities of $f'(x)$ under the three divergences.

\begin{proposition}
For all $x>0$, we have
\begin{equation}
-\frac{1}{x}<\log\frac{2x}{1+x}<\log x+1.
\label{eq:ordering_final}
\end{equation}
\end{proposition}
\begin{proof}
Firstly, we consider left part of this inequality. Define $g(x):=\log\frac{2x}{1+x}+\frac{1}{x}$ for $x>0$. The derivative of $g(x)$ can be expressed as:

\begin{equation}
g'(x)=\left(\frac{1}{x}-\frac{1}{1+x}\right)-\frac{1}{x^2}=-\frac{1}{x^2(1+x)}<0.
\end{equation}
Hence, $g(x)$ is strictly decreasing. Furthermore, we have that:

\begin{equation}
\lim_{x\to\infty}g(x)=\lim_{x\to\infty}\left(\log\frac{2x}{1+x}+\frac{1}{x}\right)=\log2>0.
\end{equation}
Therefore, for all $x>0$, we have $g(x)>0$, which means that:
\begin{equation}
\log\frac{2x}{1+x}>-\frac{1}{x}.
\label{eq:proposition_left_method}
\end{equation}
Next, consider the right part of this inequality. Since $x>0$, we have $x>\frac{2}{e}-1$, implying that $2<e(x+1)$, and thus we have $2x/(1+x)<ex$. By applying the logarithm to both sides, we can obtain that:
\begin{equation}
\log\frac{2x}{1+x}<\log x+1.
\label{eq:proposition_right_method}
\end{equation}
Combining \Cref{eq:proposition_left_method} and \Cref{eq:proposition_right_method}, we readily obtain \Cref{eq:ordering_final}, which completes the proof.
\end{proof}

\emergencystretch=1em
Therefore, reverse KL divergence yields the largest derivative for over-represented tokens with $x>1$, followed by JS divergence, while forward KL divergence satisfies $f'_{\mathrm{FKL}}(x)<0$ for all $x>0$. Within the sampled-token shaping form, this implies that forward KL divergence contributes a positive offset $\eta/x$ to the merged advantage, i.e., $\hat A=A+\eta/x$, and thus behaves as an overly weak trust-region regularizer for sampled tokens. In contrast, reverse KL divergence subtracts $\eta\cdot(\log x+1)$ from the advantage, leading to an excessively strong suppression when the current policy assigns substantially higher probability than the reference policy. The JS divergence lies between these two extremes: for $x>1$, $f'_{\mathrm{JS}}(x)>0$, resulting in suppressed updates; for $x<1$, $f'_{\mathrm{JS}}(x)<0$, leading to amplified updates. Moreover, $f'_{\mathrm{JS}}(1)=0$, implying that JS divergence introduces no constant offset at the reference point under the sampled-token gradient shaping formulation. Therefore, JS divergence provides a more balanced regularization effect between policy improvement and reference-policy preservation.

\subsection{Practical JS Divergence Computation}
In implementation, rather than explicitly computing the full-vocabulary JS divergence value, we instantiate JS regularization by modifying the shaping term in the merged advantage. Specifically, for the JS divergence, the merged advantage is given by:

\begin{equation}
\hat{A}^{\mathrm{JS}}_t=A_t-\eta\cdot f'_{\mathrm{JS}}(w_t).
\label{eq:js_merged_adv_method}
\end{equation}
To ensure numerical stability, we perform all computations of log-ratio $\Delta_t$ with: 
\begin{equation}
\Delta_t=\log\pi_\theta(y_t|x,y_{<t})-\log\pi_{\text{ref}}(y_t|x,y_{<t}).
\label{eq:log_ratio_method}
\end{equation}
Consequently, the JS-specific shaping term $f'_{\mathrm{JS}}(w_t)$ is formulated as:
\begin{equation}
f'_{\mathrm{JS}}(w_t)=\log2+\Delta_t-\log(1+\exp(\Delta_t)).
\label{eq:js_log_ratio_method}
\end{equation}
During mixed-precision training, $\Delta_t$ is clipped when necessary before evaluating the log-ratio term. All other components of the GRPO pipeline, including sampling strategy, advantage computation, and optimization schedules, are kept consistent, so that comparison primarily reflects the effect of divergence choice.

\begin{table}[t]\small
\caption{Quantitative results on performance of Janus-7B and LlamaGen variants.}
\label{tab:main_results_detail}
\centering
\begin{tabular}{l|l|c|c|c|c}
\hline
Benchmark & Model & CLIPScore $\uparrow$ & HPS-v2.1 $\uparrow$ & ImageReward $\uparrow$ & VQAScore $\uparrow$ \\
\hline
\multirow{3}{*}{DPG-Bench} 
& Janus-7B-RKL & 0.3923 & 0.2737 & 0.8152 & 0.8859 \\
& Janus-7B-FKL & 0.3923 & \textbf{0.2748} & 0.8333 & 0.8849 \\
& Janus-7B-JS  & \textbf{0.3930} & \textbf{0.2748} & \textbf{0.8371} & \textbf{0.8881} \\
\hline
\multirow{3}{*}{GenAI} 
& Janus-7B-RKL & 0.3697 & 0.2828 & 1.0569 & \textbf{0.7325} \\
& Janus-7B-FKL & 0.3694 & 0.2827 & 1.0407 & 0.7289 \\
& Janus-7B-JS  & \textbf{0.3706} & \textbf{0.2834} & \textbf{1.0713} & 0.7308 \\
\hline
\multirow{3}{*}{GenEval} 
& Janus-7B-RKL & \textbf{0.3838} & 0.2872 & 1.1237 & 0.8579 \\
& Janus-7B-FKL & 0.3817 & 0.2875 & \textbf{1.1281} & 0.8548 \\
& Janus-7B-JS  & 0.3836 & \textbf{0.2881} & 1.1272 & \textbf{0.8582} \\
\hline
\multirow{3}{*}{DPG-Bench} 
& LlamaGen-RKL & 0.2725 & 0.2209 & -0.8221 & 0.6203 \\
& LlamaGen-FKL & 0.2686 & 0.2099 & -0.9455 & 0.6257 \\
& LlamaGen-JS  & \textbf{0.2739} & \textbf{0.2214} & \textbf{-0.8037} & \textbf{0.6312} \\
\hline
\multirow{3}{*}{GenAI} 
& LlamaGen-RKL & \textbf{0.3338} & \textbf{0.2737} & \textbf{0.4788} & \textbf{0.6302} \\
& LlamaGen-FKL & 0.3320 & 0.2725 & 0.4538 & 0.6295 \\
& LlamaGen-JS  & 0.3301 & 0.2735 & 0.4749 & 0.6245 \\
\hline
\multirow{3}{*}{GenEval} 
& LlamaGen-RKL & \textbf{0.3315} & 0.2712 & 0.1360 & 0.6582 \\
& LlamaGen-FKL & 0.3294 & 0.2704 & 0.1487 & 0.6619 \\
& LlamaGen-JS  & 0.3268 & \textbf{0.2719} & \textbf{0.1580} & \textbf{0.6692} \\
\hline
\end{tabular}
\end{table}

\begin{figure}[t]
\centering
\includegraphics[width=0.8\textwidth]{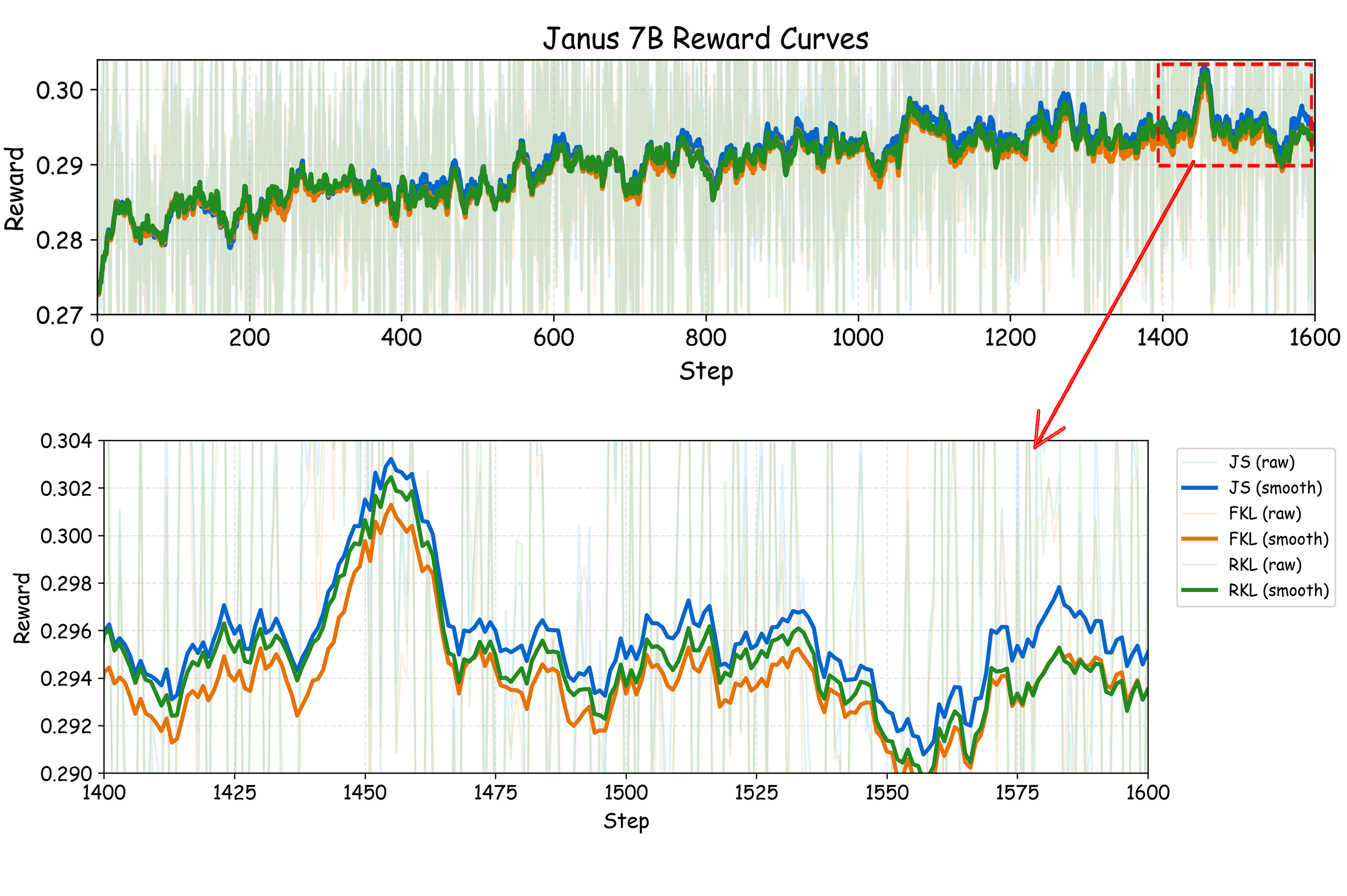}

\caption{Training reward dynamics under varying divergence regularization, showing HPS-v2 reward trajectories of Janus-7B over 1600 steps.}
\label{fig:reward_curves}
\end{figure}

\begin{table}[t]
\caption{Quantitative results on generative diversity of Janus-7B and LlamaGen variants.}
\label{tab:image_diversity_detail}
\centering
\resizebox{\textwidth}{!}{
\begin{tabular}{l|l|c|c|c|c|c|c|c|c}
\hline
Benchmark & Model & 1D Ent. $\uparrow$ & 2D Ent. $\uparrow$ & FSIM $\downarrow$ & I-I CLIP $\downarrow$ & LPIPS $\uparrow$ & RMSE $\uparrow$ & PSNR $\downarrow$ & SSIM $\downarrow$ \\
\hline
\multirow{6}{*}{DPG-Bench} 
& Janus-7B-RKL & 7.3708 & 13.6531 & \textbf{0.3301} & \textbf{0.9015} & 0.4751 & \textbf{0.0171} & 35.62 & 0.8065 \\
& Janus-7B-JS  & \textbf{7.3867} & \textbf{13.6849} & 0.3313 & 0.9031 & 0.4724 & 0.0170 & 35.63 & 0.8075 \\
& Janus-7B-FKL & 7.3805 & 13.6670 & 0.3306 & 0.9039 & \textbf{0.4760} & \textbf{0.0171} & \textbf{35.61} & \textbf{0.8064} \\
& LlamaGen-RKL & 7.4510 & 14.0465 & \textbf{0.3108} & \textbf{0.7782} & \textbf{0.5594} & \textbf{0.0216} & \textbf{33.42} & \textbf{0.7303} \\
& LlamaGen-JS  & \textbf{7.4617} & \textbf{14.0582} & 0.3160 & 0.7906 & 0.5443 & 0.0211 & 33.59 & 0.7412 \\
& LlamaGen-FKL & 7.3035 & 13.6840 & 0.3154 & 0.7833 & 0.5275 & 0.0206 & 33.82 & 0.7458 \\
\hline
\multirow{6}{*}{GenAI} 
& Janus-7B-RKL & 7.5014 & 13.6044 & \textbf{0.3566} & 0.9074 & 0.4173 & 0.0146 & 37.05 & 0.8517 \\
& Janus-7B-JS  & \textbf{7.5081} & \textbf{13.6332} & 0.3569 & 0.9078 & \textbf{0.4195} & \textbf{0.0147} & \textbf{36.95} & \textbf{0.8498} \\
& Janus-7B-FKL & 7.5067 & 13.6063 & 0.3571 & \textbf{0.9066} & 0.4185 & 0.0146 & 37.00 & 0.8505 \\
& LlamaGen-RKL & \textbf{7.3466} & \textbf{13.6565} & \textbf{0.3018} & 0.8438 & \textbf{0.5683} & \textbf{0.0214} & \textbf{33.45} & 0.7341 \\
& LlamaGen-JS  & 7.3404 & 12.7302 & 0.3045 & 0.8418 & 0.5601 & 0.0211 & 33.63 & 0.7430 \\
& LlamaGen-FKL & 7.2720 & 13.4068 & 0.3031 & \textbf{0.8417} & 0.5584 & 0.0211 & 33.50 & \textbf{0.7332} \\
\hline
\multirow{6}{*}{GenEval} 
& Janus-7B-RKL & 6.9902 & 12.1728 & \textbf{0.3238} & \textbf{0.9154} & \textbf{0.4178} & \textbf{0.0167} & \textbf{35.92} & \textbf{0.8226} \\
& Janus-7B-JS  & 6.9868 & 12.1632 & 0.3260 & 0.9192 & 0.4130 & 0.0166 & 35.99 & 0.8258 \\
& Janus-7B-FKL & \textbf{6.9967} & \textbf{12.1868} & 0.3272 & 0.9179 & 0.4093 & \textbf{0.0167} & 36.02 & 0.8248 \\
& LlamaGen-RKL & \textbf{6.9604} & 12.7253 & 0.2849 & 0.8175 & \textbf{0.5623} & 0.0232 & 32.87 & 0.7235 \\
& LlamaGen-JS  & 6.9595 & \textbf{12.7302} & 0.2864 & 0.8254 & 0.5509 & 0.0228 & 32.99 & 0.7313 \\
& LlamaGen-FKL & 6.6873 & 12.1256 & \textbf{0.2840} & \textbf{0.8160} & 0.5522 & \textbf{0.0236} & \textbf{32.70} & \textbf{0.7198} \\
\hline
\end{tabular}
}
\end{table}
\section{Experiments}

\subsection{Experimental Setup}
We train three full-parameter RL variants of LlamaGen~\cite{sun2024llamagen} and Janus-7B~\cite{wu2025janus}, differing only in the reference-policy regularizer: reverse KL, forward KL, and JS divergence.  All variants use HPS-v2.1~\cite{wu2023hpsv2} as the online reward model and are trained on four NVIDIA A100 80GB GPUs. Janus is trained for 1,600 steps with a learning rate of $1\times10^{-6}$, batch size 2, and a linear schedule; LlamaGen is trained for 900 steps with a learning rate of $1\times10^{-5}$, batch size 32, and a cosine schedule. We evaluate on DPG-Bench~\cite{hu2024ella}, GenAI-Bench~\cite{li2024genaibench}, and GenEval~\cite{ghosh2023geneval}.

For performance evaluation, we report CLIPScore~\cite{radford2021learning}, HPS-v2.1~\cite{wu2023human}, ImageReward~\cite{xu2023imagereward}, and VQAScore~\cite{lin2024evaluating}, where higher values indicate better optimization performance. For diversity analysis, we follow the generation diversity evaluation protocol in work~\cite{sun2025generalizing}, reporting Image--Image CLIP~\cite{radford2021learning}, 1D Entropy~\cite{shannon1948mathematical}, 2D Entropy~\cite{abutaleb1989automatic}, LPIPS~\cite{zhang2018unreasonable}, RMSE~\cite{chai2014root}, PSNR~\cite{huynh2008scope}, SSIM~\cite{wang2004image}, and FSIM~\cite{zhang2011fsim}. These complementary metrics jointly characterize the trade-off between optimization performance and generation diversity.

Table~\ref{tab:image_diversity_detail} summarizes the averaged diversity metrics, where arrows indicate the direction favoring greater diversity. These metrics reveal a pronounced trade-off across divergences. RKL is skewed heavily toward extreme pairwise discrepancy: the lowest FSIM and Image--Image CLIP on DPG-Bench and GenEval Janus-7B are achieved by RKL, indicating the strongest feature-level and semantic variation; meanwhile, the highest LPIPS coupled with the lowest PSNR and SSIM are recorded by RKL on DPG-Bench of LlamaGen, reflecting the most pronounced perceptual and structural deviation. FKL yields an uneven profile, topping 1D and 2D Entropy on GenEval of Janus while collapsing to the lowest entropy and the poorest pairwise-difference performance on DPG-Bench of LlamaGen. JS demonstrates a more balanced advantage, achieving the highest 1D or 2D Entropy in seven of the twelve entropy entries and ranking second in two others. Rather than pushing any single sample-difference metric to an extreme, JS maintains well-rounded scores across all dimensions. Coupled with the best average HPS, ImageReward, and VQAScore in Table~\ref{tab:main_results_detail}, JS therefore realizes a balanced performance-diversity trade-off: optimization performance is improved without collapsing into either excessive variation-seeking or undue conservatism, preserving broad visual-distribution coverage and non-trivial spatial complexity.

\subsection{Training Dynamics}
As illustrated in Figure~\ref{fig:reward_curves}, the training reward profiles show that JS achieves strong and competitive optimization performance, maintaining a slight reward advantage on Janus-7B in the later stage of training. RKL follows a steadier conservative path, while FKL underperforms. These empirical trends are consistent with our gradient analysis, which indicates that distinct $f'(w)$ terms effectively modulate token-level advantages, thereby dictating divergent optimization paths.

\section{Conclusion}

In this work, we revisit GRPO-style online reinforcement learning for autoregressive text-to-image alignment through a unified $f$-divergence perspective. Our analysis shows that the choice of divergence is not merely a regularization detail, but a key factor that fundamentally shapes token-level optimization dynamics. Within the sampled-token shaping formulation, JS divergence emerges as a particularly effective choice. Empirical results on LlamaGen and Janus-7B support our analysis, demonstrating that JS divergence regularization consistently delivers strong or competitive optimization performance across metrics while preserving generation diversity. These findings also highlight the importance of carefully selecting divergence objectives for effective optimization in autoregressive T2I generation.

\bibliographystyle{plainnat}
\bibliography{references}

@article{zhang2025group,
  title={Group Critical-token Policy Optimization for Autoregressive Image Generation},
  author={Zhang, Guohui and Yu, Hu and Ma, Xiaoxiao and Zhang, JingHao and Pan, Yaning and Yao, Mingde and Xiao, Jie and Huang, Linjiang and Zhao, Feng},
  journal={arXiv preprint arXiv:2509.22485},
  year={2025}
}

@article{shao2024deepseekmath,
  title={Deepseekmath: Pushing the limits of mathematical reasoning in open language models},
  author={Shao, Zhihong and Wang, Peiyi and Zhu, Qihao and Xu, Runxin and Song, Junxiao and Bi, Xiao and Zhang, Haowei and Zhang, Mingchuan and Li, YK and Wu, Yang and others},
  journal={arXiv preprint arXiv:2402.03300},
  year={2024}
}

@article{sutton1999policy,
  title={Policy gradient methods for reinforcement learning with function approximation},
  author={Sutton, Richard S and McAllester, David and Singh, Satinder and Mansour, Yishay},
  journal={Advances in neural information processing systems},
  volume={12},
  year={1999}
}

@article{thomas2017policy,
  title={Policy gradient methods for reinforcement learning with function approximation and action-dependent baselines},
  author={Thomas, Philip S and Brunskill, Emma},
  journal={arXiv preprint arXiv:1706.06643},
  year={2017}
}

@article{kullback1951information,
  title={On information and sufficiency},
  author={Kullback, Solomon and Leibler, Richard A},
  journal={The annals of mathematical statistics},
  volume={22},
  number={1},
  pages={79--86},
  year={1951},
  publisher={JSTOR}
}

@article{lin1991divergence,
  title={Divergence measures based on the Shannon entropy},
  author={Lin, Jianhua},
  journal={IEEE Transactions on Information theory},
  volume={37},
  number={1},
  pages={145--151},
  year={1991},
  publisher={IEEE}
}

@article{csiszar1967information,
  title={On information-type measure of difference of probability distributions and indirect observations},
  author={Csisz{\'a}r, Imre},
  journal={Studia Sci. Math. Hungar.},
  volume={2},
  pages={299--318},
  year={1967}
}

@article{ouyang2022training,
  title={Training language models to follow instructions with human feedback},
  author={Ouyang, Long and Wu, Jeffrey and Jiang, Xu and Almeida, Diogo and Wainwright, Carroll and Mishkin, Pamela and Zhang, Chong and Agarwal, Sandhini and Slama, Katarina and Ray, Alex and others},
  journal={Advances in neural information processing systems},
  volume={35},
  pages={27730--27744},
  year={2022}
}

@article{fan2023dpok,
  title={Dpok: Reinforcement learning for fine-tuning text-to-image diffusion models},
  author={Fan, Ying and Watkins, Olivia and Du, Yuqing and Liu, Hao and Ryu, Moonkyung and Boutilier, Craig and Abbeel, Pieter and Ghavamzadeh, Mohammad and Lee, Kangwook and Lee, Kimin},
  journal={Advances in Neural Information Processing Systems},
  volume={36},
  pages={79858--79885},
  year={2023}
}

@article{schulman2017proximal,
  title={Proximal policy optimization algorithms},
  author={Schulman, John and Wolski, Filip and Dhariwal, Prafulla and Radford, Alec and Klimov, Oleg},
  journal={arXiv preprint arXiv:1707.06347},
  year={2017}
}

@article{goodfellow2014generative,
  title={Generative adversarial nets},
  author={Goodfellow, Ian J and Pouget-Abadie, Jean and Mirza, Mehdi and Xu, Bing and Warde-Farley, David and Ozair, Sherjil and Courville, Aaron and Bengio, Yoshua},
  journal={Advances in neural information processing systems},
  volume={27},
  year={2014}
}

@inproceedings{sun2025generalizing,
  title={Generalizing alignment paradigm of text-to-image generation with preferences through f-divergence minimization},
  author={Sun, Haoyuan and Xia, Bo and Chang, Yongzhe and Wang, Xueqian},
  booktitle={Proceedings of the AAAI Conference on Artificial Intelligence},
  volume={39},
  pages={27644--27652},
  year={2025}
}

@inproceedings{wang2024beyond,
  title={Beyond reverse kl: Generalizing direct preference optimization with diverse divergence constraints},
  author={Wang, Chaoqi and Jiang, Yibo and Yang, Chenghao and Liu, Han and Chen, Yuxin},
  booktitle={International Conference on Learning Representations},
  volume={2024},
  pages={10450--10480},
  year={2024}
}

@article{go2023aligning,
  title={Aligning language models with preferences through f-divergence minimization},
  author={Go, Dongyoung and Korbak, Tomasz and Kruszewski, Germ{\'a}n and Rozen, Jos and Ryu, Nahyeon and Dymetman, Marc},
  journal={arXiv preprint arXiv:2302.08215},
  year={2023}
}

@inproceedings{wallace2024diffusion,
  title={Diffusion model alignment using direct preference optimization},
  author={Wallace, Bram and Dang, Meihua and Rafailov, Rafael and Zhou, Linqi and Lou, Aaron and Purushwalkam, Senthil and Ermon, Stefano and Xiong, Caiming and Joty, Shafiq and Naik, Nikhil},
  booktitle={Proceedings of the IEEE/CVF Conference on Computer Vision and Pattern Recognition},
  pages={8228--8238},
  year={2024}
}

@inproceedings{yang2024using,
  title={Using human feedback to fine-tune diffusion models without any reward model},
  author={Yang, Kai and Tao, Jian and Lyu, Jiafei and Ge, Chunjiang and Chen, Jiaxin and Shen, Weihan and Zhu, Xiaolong and Li, Xiu},
  booktitle={Proceedings of the IEEE/CVF Conference on Computer Vision and Pattern Recognition},
  pages={8941--8951},
  year={2024}
}

@article{liang2024step,
  title={Step-aware preference optimization: Aligning preference with denoising performance at each step},
  author={Liang, Zhanhao and Yuan, Yuhui and Gu, Shuyang and Chen, Bohan and Hang, Tiankai and Li, Ji and Zheng, Liang},
  journal={arXiv preprint arXiv:2406.04314},
  volume={2},
  number={5},
  pages={7},
  year={2024}
}

@article{li2024aligning,
  title={Aligning diffusion models by optimizing human utility},
  author={Li, Shufan and Kallidromitis, Konstantinos and Gokul, Akash and Kato, Yusuke and Kozuka, Kazuki},
  journal={Advances in Neural Information Processing Systems},
  volume={37},
  pages={24897--24925},
  year={2024}
}

@article{ghosh2023geneval,
  title={Geneval: An object-focused framework for evaluating text-to-image alignment},
  author={Ghosh, Dhruba and Hajishirzi, Hannaneh and Schmidt, Ludwig},
  journal={Advances in Neural Information Processing Systems},
  volume={36},
  pages={52132--52152},
  year={2023}
}

@article{li2024genaibench,
  title={Genai-bench: Evaluating and improving compositional text-to-visual generation},
  author={Li, Baiqi and Lin, Zhiqiu and Pathak, Deepak and Li, Jiayao and Fei, Yixin and Wu, Kewen and Ling, Tiffany and Xia, Xide and Zhang, Pengchuan and Neubig, Graham and others},
  journal={arXiv preprint arXiv:2406.13743},
  year={2024}
}

@article{hu2024ella,
  title={Ella: Equip diffusion models with llm for enhanced semantic alignment},
  author={Hu, Xiwei and Wang, Rui and Fang, Yixiao and Fu, Bin and Cheng, Pei and Yu, Gang},
  journal={arXiv preprint arXiv:2403.05135},
  year={2024}
}

@article{shannon1948mathematical,
  title={A mathematical theory of communication},
  author={Shannon, Claude Elwood},
  journal={The Bell system technical journal},
  volume={27},
  number={3},
  pages={379--423},
  year={1948},
  publisher={Nokia Bell Labs}
}

@article{abutaleb1989automatic,
  title={Automatic thresholding of gray-level pictures using two-dimensional entropy},
  author={Abutaleb, Ahmed S},
  journal={Computer vision, graphics, and image processing},
  volume={47},
  number={1},
  pages={22--32},
  year={1989},
  publisher={Elsevier}
}

@article{zhang2011fsim,
  title={FSIM: A feature similarity index for image quality assessment},
  author={Zhang, Lin and Zhang, Lei and Mou, Xuanqin and Zhang, David},
  journal={IEEE transactions on Image Processing},
  volume={20},
  number={8},
  pages={2378--2386},
  year={2011},
  publisher={IEEE}
}

@inproceedings{radford2021learning,
  title={Learning transferable visual models from natural language supervision},
  author={Radford, Alec and Kim, Jong Wook and Hallacy, Chris and Ramesh, Aditya and Goh, Gabriel and Agarwal, Sandhini and Sastry, Girish and Askell, Amanda and Mishkin, Pamela and Clark, Jack and others},
  booktitle={International conference on machine learning},
  pages={8748--8763},
  year={2021},
  organization={PmLR}
}

@inproceedings{wu2025janus,
  title={Janus: Decoupling visual encoding for unified multimodal understanding and generation},
  author={Wu, Chengyue and Chen, Xiaokang and Wu, Zhiyu and Ma, Yiyang and Liu, Xingchao and Pan, Zizheng and Liu, Wen and Xie, Zhenda and Yu, Xingkai and Ruan, Chong and others},
  booktitle={Proceedings of the Computer Vision and Pattern Recognition Conference},
  pages={12966--12977},
  year={2025}
}

@article{sun2024llamagen,
  title={Autoregressive model beats diffusion: Llama for scalable image generation},
  author={Sun, Peize and Jiang, Yi and Chen, Shoufa and Zhang, Shilong and Peng, Bingyue and Luo, Ping and Yuan, Zehuan},
  journal={arXiv preprint arXiv:2406.06525},
  year={2024}
}

@article{wu2023hpsv2,
  title={Human preference score v2: A solid benchmark for evaluating human preferences of text-to-image synthesis},
  author={Wu, Xiaoshi and Hao, Yiming and Sun, Keqiang and Chen, Yixiong and Zhu, Feng and Zhao, Rui and Li, Hongsheng},
  journal={arXiv preprint arXiv:2306.09341},
  year={2023}
}

@inproceedings{zhang2018unreasonable,
  title={The unreasonable effectiveness of deep features as a perceptual metric},
  author={Zhang, Richard and Isola, Phillip and Efros, Alexei A and Shechtman, Eli and Wang, Oliver},
  booktitle={Proceedings of the IEEE conference on computer vision and pattern recognition},
  pages={586--595},
  year={2018}
}

@article{chai2014root,
  title={Root mean square error (RMSE) or mean absolute error (MAE)?--Arguments against avoiding RMSE in the literature},
  author={Chai, Tianfeng and Draxler, Roland R},
  journal={Geoscientific model development},
  volume={7},
  number={3},
  pages={1247--1250},
  year={2014},
  publisher={Copernicus Publications G{\"o}ttingen, Germany}
}

@article{huynh2008scope,
  title={Scope of validity of PSNR in image/video quality assessment},
  author={Huynh-Thu, Quan and Ghanbari, Mohammed},
  journal={Electronics letters},
  volume={44},
  number={13},
  pages={800--801},
  year={2008},
  publisher={IET}
}

@article{wang2004image,
  title={Image quality assessment: from error visibility to structural similarity},
  author={Wang, Zhou and Bovik, Alan C and Sheikh, Hamid R and Simoncelli, Eero P},
  journal={IEEE transactions on image processing},
  volume={13},
  number={4},
  pages={600--612},
  year={2004},
  publisher={IEEE}
}

@inproceedings{wu2023human,
  title={Human preference score: Better aligning text-to-image models with human preference},
  author={Wu, Xiaoshi and Sun, Keqiang and Zhu, Feng and Zhao, Rui and Li, Hongsheng},
  booktitle={Proceedings of the IEEE/CVF International Conference on Computer Vision},
  pages={2096--2105},
  year={2023}
}

@article{xu2023imagereward,
  title={Imagereward: Learning and evaluating human preferences for text-to-image generation},
  author={Xu, Jiazheng and Liu, Xiao and Wu, Yuchen and Tong, Yuxuan and Li, Qinkai and Ding, Ming and Tang, Jie and Dong, Yuxiao},
  journal={Advances in Neural Information Processing Systems},
  volume={36},
  pages={15903--15935},
  year={2023}
}

@inproceedings{lin2024evaluating,
  title={Evaluating text-to-visual generation with image-to-text generation},
  author={Lin, Zhiqiu and Pathak, Deepak and Li, Baiqi and Li, Jiayao and Xia, Xide and Neubig, Graham and Zhang, Pengchuan and Ramanan, Deva},
  booktitle={European Conference on Computer Vision},
  pages={366--384},
  year={2024},
  organization={Springer}
}

@inproceedings{sun2024generalizing,
  title={Generalizing offline alignment theoretical paradigm with diverse divergence constraints},
  author={Sun, Haoyuan and Zheng, Yuxin and Zhao, Yifei and Chang, Yongzhe and Wang, Xueqian},
  booktitle={ICML 2024 Workshop on Models of Human Feedback for AI Alignment},
  year={2024}
}

@inproceedings{sun2025positive,
  title={Positive enhanced preference alignment for text-to-image models},
  author={Sun, Haoyuan and Xia, Bo and Zhao, Yifei and Chang, Yongzhe and Wang, Xueqian},
  booktitle={ICASSP 2025-2025 IEEE International Conference on Acoustics, Speech and Signal Processing (ICASSP)},
  year={2025},
  organization={IEEE}
}

@inproceedings{sun2025identical,
  title={Identical human preference alignment paradigm for text-to-image models},
  author={Sun, Haoyuan and Xia, Bo and Zhao, Yifei and Chang, Yongzhe and Wang, Xueqian},
  booktitle={ICASSP 2025-2025 IEEE International Conference on Acoustics, Speech and Signal Processing (ICASSP)},
  year={2025},
  organization={IEEE}
}

@inproceedings{bai2025entropy,
  title={Entropy-Aware Preference Alignment for Diffusion-Based Text-to-Image Generation},
  author={Bai, Hannan and Sun, Haoyuan and Du, Yuncheng},
  booktitle={Chinese Conference on Pattern Recognition and Computer Vision (PRCV)},
  pages={373--387},
  year={2025},
  organization={Springer}
}

@article{lv2025hidden,
  title={The hidden link between rlhf and contrastive learning},
  author={Lv, Xufei and Chen, Kehai and Sun, Haoyuan and Bai, Xuefeng and Zhang, Min and Liu, Houde},
  journal={arXiv preprint arXiv:2506.22578},
  year={2025}
}

@article{luo2026reinforcement,
  title={Reinforcement learning meets masked generative models: Mask-grpo for text-to-image generation},
  author={Luo, Yifu and Hu, Xinhao and Fan, Keyu and Sun, Haoyuan and Chen, Zeyu and Xia, Bo and Zhang, Tiantian and Chang, Yongzhe and Wang, Xueqian},
  journal={Advances in Neural Information Processing Systems},
  volume={38},
  pages={108460--108485},
  year={2026}
}

@inproceedings{fang2026viss,
  title={Viss-r1: Self-supervised reinforcement video reasoning},
  author={Fang, Bo and Song, Yuxin and Sun, Haoyuan and Zhang, Xinyao and Wu, Qiangqiang and Wu, Wenhao and Chan, Antoni B},
  booktitle={Proceedings of the IEEE/CVF Conference on Computer Vision and Pattern Recognition},
  pages={11190--11200},
  year={2026}
}

@article{luo2025sample,
  title={Sample By Step, Optimize By Chunk: Chunk-Level GRPO For Text-to-Image Generation},
  author={Luo, Yifu and Du, Penghui and Li, Bo and Du, Sinan and Zhang, Tiantian and Chang, Yongzhe and Wu, Kai and Gai, Kun and Wang, Xueqian},
  journal={arXiv preprint arXiv:2510.21583},
  year={2025}
}

@article{sun2026power,
  title={Power Reinforcement Post-Training of Text-to-Image Models with Super-Linear Advantage Shaping},
  author={Sun, Haoyuan and Wang, Jing and Song, Yuxin and Lu, Yu and Fang, Bo and Luo, Yifu and Yin, Jun and Zeng, Pengyu and Zhang, Miao and Zhang, Tiantian and others},
  journal={arXiv preprint arXiv:2605.10937},
  year={2026}
}

@article{sun2025reinforcement,
  title={Reinforcement fine-tuning powers reasoning capability of multimodal large language models},
  author={Sun, Haoyuan and Wu, Jiaqi and Xia, Bo and Luo, Yifu and Zhao, Yifei and Qin, Kai and Lv, Xufei and Zhang, Tiantian and Chang, Yongzhe and Wang, Xueqian},
  journal={arXiv preprint arXiv:2505.18536},
  year={2025}
}

@article{xia2025delay,
  title={A delay-robust method for enhanced real-time reinforcement learning},
  author={Xia, Bo and Sun, Haoyuan and Yuan, Bo and Li, Zhiheng and Liang, Bin and Wang, Xueqian},
  journal={Neural Networks},
  volume={181},
  pages={106769},
  year={2025},
  publisher={Elsevier}
}

@article{suncalibration,
  title={Calibration Enhanced Decision Maker: Towards Trustworthy Sequential Decision-Making with Large Sequence Models},
  author={Sun, Haoyuan and Xia, Bo and Luo, Yifu and Zhang, Tiantian and Wang, Xueqian},
  journal={Transactions on Machine Learning Research},
  year={2026}
}

@inproceedings{yin2025floorplan,
  title={Floorplan-llama: Aligning architects’ feedback and domain knowledge in architectural floor plan generation},
  author={Yin, Jun and Zeng, Pengyu and Sun, Haoyuan and Dai, Yuqin and Zheng, Han and Zhang, Miao and Zhang, Yachao and Lu, Shuai},
  booktitle={Proceedings of the 63rd Annual Meeting of the Association for Computational Linguistics (Volume 1: Long Papers)},
  pages={6640--6662},
  year={2025}
}

@inproceedings{zeng2025mred,
  title={MRED-14: A Benchmark for Low-Energy Residential Floor Plan Generation with 14 Flexible Inputs},
  author={Zeng, Pengyu and Yin, Jun and Sun, Haoyuan and Dai, Yuqin and Jiang, Maowei and Zhang, Miao and Lu, Shuai},
  booktitle={Proceedings of the 33rd ACM International Conference on Multimedia},
  pages={11298--11307},
  year={2025}
}

@inproceedings{song2026cologen,
  title={Cologen: Progressive learning of concept-localization duality for unified image generation},
  author={Song, Yuxin and Lu, Yu and Sun, Haoyuan and Yao, Huanjin and Liu, Fanglong and Sun, Yifan and Feng, Haocheng and Zhou, Hang and Wang, Jingdong},
  booktitle={Proceedings of the IEEE/CVF Conference on Computer Vision and Pattern Recognition},
  pages={14724--14734},
  year={2026}
}

@article{yin2026ai,
  title={AI-empowered prediction of office building energy use from single-view conceptual images for early-stage design},
  author={Yin, Jun and Zeng, Pengyu and Huang, Yujian and Sun, Haoyuan and Hao, Tianze and Lu, Shuai and others},
  journal={Applied Energy},
  volume={406},
  pages={127289},
  year={2026},
  publisher={Elsevier}
}

\end{document}